%% file: main.tex
\PassOptionsToPackage{table}{xcolor}
\documentclass[sigconf]{aamas}
\input{aamas-2027-compat.tex}

\usepackage{balance}
\usepackage{placeins}
\usepackage{xurl}
\usepackage[table]{xcolor}
\usepackage{tabularx}
\newcommand{\cmark}{\ding{51}}
\newcommand{\xmark}{\ding{55}}
\makeatletter
\gdef\@copyrightpermission{
  \begin{minipage}{0.2\columnwidth}
   \href{https://creativecommons.org/licenses/by/4.0/}{\includegraphics[width=0.90\textwidth]{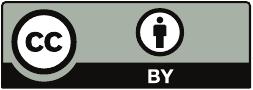}}
  \end{minipage}\hfill
  \begin{minipage}{0.8\columnwidth}
   \href{https://creativecommons.org/licenses/by/4.0/}{This work is licensed under a Creative Commons Attribution International 4.0 License.}
  \end{minipage}
  
}
\makeatother

\setcopyright{ifaamas}
\acmConference[AAMAS 2027]{Proc.\@ of the 26th International Conference
on Autonomous Agents and Multiagent Systems (AAMAS 2027)}{3--7 May 2027}{Hanoi, Vietnam}{M.~Baldoni, F.~Fang, W.~Yeoh, N.~Yorke-Smith (eds.)}
\copyrightyear{2027}
\acmYear{2027}
\acmDOI{}
\acmISBN{}

\acmSubmissionID{}
\submissionType{Research Paper Track}

\title{Heavy-Tailed Memory Traces in Long-Horizon Language Agents}

\author{Xinyuan Song}
\affiliation{
  \institution{Emory University}
  \city{Atlanta}
  \country{USA}
}
\email{xinyuan.song@emory.edu}
\author{Zekun Cai}
\affiliation{
  \institution{The University of Tokyo}
  \city{Tokyo}
  \country{Japan}
}
\email{caizekun@csis.u-tokyo.ac.jp}

\AtBeginEnvironment{figure}{%
  \Description{Experimental results and analyses illustrating the world-model dynamics, critical behavior, error propagation, avalanche statistics, and intervention effects studied in this work.}%
}
\AtBeginEnvironment{figure*}{%
  \Description{Experimental results and analyses illustrating the world-model dynamics, critical behavior, error propagation, avalanche statistics, and intervention effects studied in this work.}%
}

\begin{abstract}
Long-horizon language agents are becoming increasingly important for tasks that require persistent interaction, planning, and reasoning over extended trajectories. External memory plays a central role in these agents by preserving past states, transitions, and experiences beyond the limited context window. Existing memory systems, however, are typically evaluated mainly by task success and token cost, with much less attention to how memory usage is distributed across stored information. We show that repeated retrieval can produce heavy-tailed memory usage, where a small set of states is accessed frequently while many others are rarely retrieved and more prone to prediction errors. We study this phenomenon through a systematic tail audit across retrieval policies and graph structures. We find that memory concentration is reproducible but strongly policy-dependent, ranging from log-normal-compatible patterns under random walks to truncated-power-law-compatible heavy tails under semantic LLM policies. Motivated by these findings, we propose \textbf{Core--Tail Memory Controller (CTMC)}, a rank-based memory controller that allocates prompt budget with a single exponent $\tau$ while preserving summarized information from the tail. Across Synthetic Graph World and ALFWorld, CTMC substantially reduces memory usage while preserving performance, and on LongMemEval it reduces prompt tokens by 24.48\% while maintaining aggregate accuracy. These results show that heavy-tailed memory usage is both a measurable property of long-horizon agent memory and a practical signal for improving memory efficiency.
\end{abstract}


\newtheorem{theorem}{Theorem}[section]

\newtheorem{definition}[theorem]{Definition}

\AtEndPreamble{%
  \hypersetup{
    pdfauthor={},
    colorlinks=true,
    linkcolor=blue,
    citecolor=black,
    urlcolor=blue,
    pdfborder={0 0 0},
    pdfborderstyle={}
  }%
}

\newsavebox{\ManuscriptLinkBox}
\DeclareRobustCommand{\ManuscriptLinkFrame}[2]{%
  \begingroup
    \setlength{\fboxsep}{0.3pt}%
    \setlength{\fboxrule}{0.6pt}%
    \sbox{\ManuscriptLinkBox}{#2}%
    \leavevmode
    \rlap{%
      \kern-\dimexpr\fboxsep+\fboxrule\relax
      \smash{\textcolor{#1}{\fbox{\phantom{\usebox{\ManuscriptLinkBox}}}}}%
    }%
    \usebox{\ManuscriptLinkBox}%
  \endgroup
}
\AddToHook{begindocument/end}{%
  \let\ManuscriptOriginalCiteNumFont\citenumfont
  \renewcommand*{\citenumfont}[1]{%
    \ManuscriptLinkFrame{green}{\ManuscriptOriginalCiteNumFont{#1}}%
  }%
  \let\ManuscriptOriginalEqref\eqref
  \renewcommand*{\eqref}[1]{%
    \begingroup
      \hypersetup{linkcolor=black}%
      \ManuscriptLinkFrame{red}{\ManuscriptOriginalEqref{#1}}%
    \endgroup
  }%
}

\begin{document}
\maketitle

\section{Introduction}

Long-horizon language agents are becoming increasingly important for tasks that require persistent interaction, planning, and reasoning over extended trajectories. Unlike single-turn language-model inference, these agents must preserve and reuse useful information across observations, actions, tool calls, and multiple interaction sessions. External memory has therefore become a key component of long-horizon agents, enabling relevant past experience to remain accessible beyond the current context window \citep{packer2023memgpt,park2023generative}. Systems such as MemoryBank further highlight the importance of continuously updating and selectively retaining information during sustained interaction \citep{zhong2024memorybank_new}, while more recent agentic memory architectures organize accumulated experience into dynamically evolving structures rather than treating memory as a static archive \citep{xu2025amem_new}.

As interaction histories grow, memory systems must selectively retain and retrieve useful information under a limited context budget. Prior work has addressed this need through mechanisms ranging from episodic and persistent memory to structured retrieval and adaptive memory organization \citep{shinn2023reflexion,zhong2024memorybank,gutierrez2024hipporag,anokhin2024arigraph,xu2025amem}. Long-term memory benchmarks further demonstrate the importance of accessing relevant information across sessions and temporal dependencies, rather than simply preserving increasingly long interaction histories \citep{wu2025longmemeval_new}. These systems are typically evaluated by downstream performance, retrieval quality, and context or token cost. Yet these metrics reveal little about how retrieval effort is distributed across stored memories over long trajectories. Some memories may be repeatedly accessed while others are rarely revisited, but the extent and structure of this imbalance remain poorly understood. In particular, it is unclear whether such concentration is a consistent property of long-horizon memory, how it changes across retrieval policies and memory structures, and whether it can inform more efficient memory allocation.

We study this distribution and find that repeated retrieval can produce heavy-tailed memory usage. A small set of states receives a large fraction of memory access, while many others are rarely retrieved. This imbalance matters because rarely retrieved states receive less representation in the context used for subsequent prediction and decision making, allowing errors to accumulate. Heavy-tailed memory usage therefore offers a different view of memory efficiency: beyond how many tokens are consumed, we ask \emph{where those tokens are repeatedly spent}.

\begin{figure*}[!ht]
  \centering
  \includegraphics[width=0.9\linewidth]{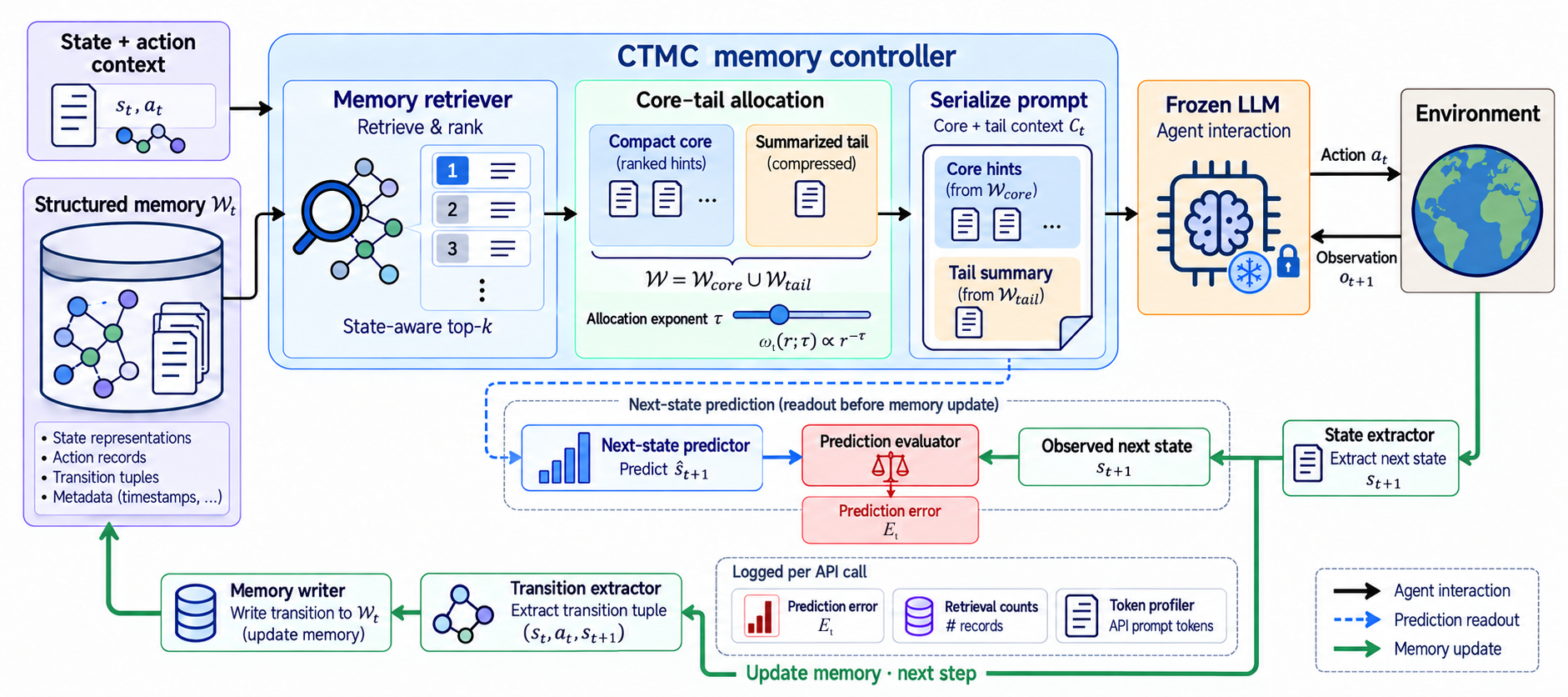}
  
\caption{\textbf{CTMC agent loop.} (a) Retrieve and rank candidates from structured memory. (b) Allocate context between a compact core and a summarized tail using rank-based CTMC. (c) Construct the LLM context, evaluate next-state predictions, and update memory from new interactions.}
  \label{fig:intuition}
  
\end{figure*}

Heavy-tailed distributions are common in networks and systems shaped by heterogeneous interactions \citep{newman2005powerlaws}, and similar concentration can emerge in agent memory as a small set of states receives a large fraction of retrievals. We examine this phenomenon using a tail audit across retrieval policies and graph topologies \citep{clauset2009powerlaw,alstott2014powerlaw}. Random-walk agents already exhibit substantial concentration, with traces that are primarily log-normal-compatible, while semantic LLM policies produce stronger heavy-tailed behavior and the clearest truncated-power-law-compatible traces. Additional topology and retrieval-mechanism controls show that this concentration is reproducible across memory settings \citep{virkar2014binned}. We capture this structure with a rank-based memory allocation,
\begin{equation}
\omega_t(r;\tau)\propto r^{-\tau},
\label{eq:intro_allocation}
\end{equation}
where $r$ denotes retrieval rank, $\omega_t(r;\tau)$ denotes the allocation share assigned to rank $r$ at step $t$, and $\tau>0$ controls the concentration toward high-priority memory items.

Motivated by this observed structure in memory use, we design \textbf{Core--Tail Memory Controller (CTMC)}, a rank-based memory controller that uses the measured memory distribution to allocate a limited context budget. As illustrated in Figure~\ref{fig:intuition}, CTMC ranks stored memories according to relevance and transition support and allocates the prompt budget following Eq.~\eqref{eq:intro_allocation}. Increasing $\tau$ concentrates more of the available budget on highly ranked memories, while CTMC preserves summarized information from the tail rather than discarding it. CTMC thus controls how memory is distributed within the context without retraining the underlying language model or introducing a separate learned dynamics model.

We evaluate CTMC across synthetic graph dynamics, embodied text environments, and long-term conversational memory. On Synthetic Graph World, CTMC reduces prompt-token usage and tail prediction error while preserving full state and transition coverage. On ALFWorld \citep{shridhar2020alfworld}, CTMC consistently reduces prompt-token usage across different ReAct configurations \citep{yao2022react}. On LongMemEval \citep{wu2025longmemeval_new}, which evaluates long-term memory across information extraction, multi-session reasoning, temporal reasoning, knowledge updates, and abstention, CTMC reduces prompt tokens by 24.48\% while maintaining aggregate accuracy. Together, these results show that the observed memory-usage structure can be directly used to allocate context more efficiently across different long-horizon settings.

Our main contributions are summarized as follows:
\begin{itemize}
    \item We characterize \textbf{heavy-tailed memory usage} in long-horizon language agents, showing reproducible concentration across retrieval and topology controls, with the clearest truncated-power-law-compatible traces under semantic LLM policies.
    
    \item We introduce \textbf{Core-Tail Memory Controller}, a rank-based controller that converts memory-usage structure into a prompt-allocation rule governed by a single parameter $\tau$, with formal monotonicity guarantees.
    
    \item We validate CTMC on \textbf{Synthetic Graph World, ALFWorld, and LongMemEval}, showing consistent prompt-token savings while preserving overall performance, including a 24.48\% reduction on LongMemEval with aggregate accuracy parity.
\end{itemize}

\section{Related Work}

\begin{table*}[!ht]
\centering
\caption{\textbf{Comparison with related agent-memory systems.}
Methods are categorized by persistent-memory construction, retrieval control, explicit budget allocation, tail-aware representation, distributional diagnosis, and matched-replay analysis.}
\label{tab:related_work_comparison}

\begingroup

\resizebox{0.95\textwidth}{!}{%
\begin{tabular}{
@{}
>{\raggedright\arraybackslash}p{0.15\textwidth}
>{\centering\arraybackslash}p{0.11\textwidth}
>{\centering\arraybackslash}p{0.11\textwidth}
>{\centering\arraybackslash}p{0.12\textwidth}
>{\centering\arraybackslash}p{0.12\textwidth}
>{\centering\arraybackslash}p{0.13\textwidth}
>{\centering\arraybackslash}p{0.13\textwidth}
@{}
}
\toprule
\textbf{Method}
& \textbf{Persistent memory}
& \textbf{Retrieval control}
& \textbf{Budget allocation}
& \textbf{Tail-aware}
& \textbf{Distributional diagnosis}
& \textbf{Matched replay}
\\
\midrule

AriGraph \citep{anokhin2024arigraph}
& \cmark & \cmark & \xmark & \xmark & \xmark & \xmark \\

A-Mem \citep{xu2025amem}
& \cmark & \cmark & \xmark & \xmark & \xmark & \xmark \\

Zep \citep{rasmussen2025zep}
& \cmark & \cmark & \xmark & \xmark & \xmark & \xmark \\

Mem0 \citep{chhikara2025mem0}
& \cmark & \cmark & \xmark & \xmark & \xmark & \xmark \\

\midrule

LightMem \citep{fang2026lightmem}
& \cmark & \cmark & \cmark & \cmark & \xmark & \xmark \\

ACON \citep{acon2025}
& \xmark & \cmark & \cmark & \xmark & \xmark & \xmark \\

AgeMem \citep{yu2026agemem}
& \cmark & \cmark & \cmark & \cmark & \xmark & \xmark \\

\midrule
\rowcolor{gray!15}
\textbf{CTMC (ours)}
& \cmark & \cmark & \cmark & \cmark & \cmark & \cmark \\

\bottomrule
\end{tabular}%
}
\endgroup

\end{table*}

\paragraph{Long-term memory for language agents.}
External memory is central to long-horizon language agents. MemGPT \citep{packer2023memgpt} formulates memory management as operating-system-style paging, Generative Agents \citep{park2023generative} retrieve memories using recency, importance, and relevance, and Reflexion \citep{shinn2023reflexion} stores verbal feedback for later attempts. MemoryBank \citep{zhong2024memorybank} and A-MEM \citep{xu2025amem} further organize and update long-term experience, while AriGraph \citep{anokhin2024arigraph} and Zep \citep{rasmussen2025zep} maintain relational or temporal memory structures. Voyager \citep{wang2023voyager} similarly accumulates reusable skills through interaction. In contrast, we study how stored memories are repeatedly accessed over long trajectories and use the resulting usage structure to guide finite-budget allocation.

\paragraph{Retrieval and finite context.}
Retrieval-augmented generation (RAG) \citep{lewis2020rag}, kNN-LM \citep{khandelwal2020knnlm}, and RETRO \citep{borgeaud2022retro} use external evidence to improve prediction without updating model weights. Yet longer context does not guarantee effective use of available information: Lost in the Middle \citep{liu2024lost} shows that models can underuse information within long prompts, while StreamingLLM \citep{xiao2024streamingllm} preserves selected tokens under bounded attention. We study this problem at the memory-system level, viewing a finite context budget as an allocation problem over retrieved memories rather than simply a context-length constraint.

\paragraph{Tool use and interactive benchmarks.}
ReAct \citep{yao2022react} interleaves language reasoning with environment actions, while Toolformer \citep{schick2023toolformer}, ToolBench/ToolLLM \citep{qin2023toolllm}, API-Bank \citep{li2023apibank}, and Gorilla \citep{patil2024gorilla} study tool selection and execution. ALFWorld \citep{shridhar2020alfworld} provides embodied text interaction through TextWorld \citep{cote2018textworld}, while WebShop \citep{yao2022webshop} and HotpotQA \citep{yang2018hotpotqa} provide complementary interactive and multi-hop settings. LongMemEval \citep{wu2025longmemeval_new} evaluates long-term conversational memory, and $\tau$-bench \citep{yao2024taubench} evaluates realistic tool-agent-user interaction. These settings motivate evaluating memory inside the full agent loop; accordingly, we measure actual API-reported prompt usage and use paired comparisons rather than relying only on offline retrieval metrics.

\paragraph{Graph retrieval and retrieval concentration.}
GraphRAG \citep{edge2024graphrag}, LightRAG \citep{he2024lightrag}, and HippoRAG \citep{gutierrez2024hipporag} exploit relational structure to improve retrieval and reasoning. However, graph connectivity and top-$k$ retrieval can repeatedly expose a small subset of memory items, producing concentration similar to link-based ranking \citep{brin1998anatomy}. We account for this effect using random-walk and retrieval-mechanism controls before using the observed memory ranking for CTMC allocation.

\paragraph{Heavy-tail statistics and criticality.}
Heavy-tailed distributions occur across natural, social, and computational systems \citep{newman2005powerlaws}, including observations in neural scaling and learned models \citep{kaplan2020scaling,martin2018implicit,martin2019htsr}. Distinguishing among heavy-tailed distributions requires care because finite samples and alternative tail families can produce similar patterns \citep{clauset2009powerlaw,alstott2014powerlaw,stumpf2012critical,virkar2014binned,broido2019scale}. Evidence for criticality requires additional temporal and event-level structure \citep{bak1987soc,sethna2001crackling,beggs2003neuronal,cocchi2017criticality}. We therefore characterize memory traces as \emph{heavy-tailed}, \emph{log-normal-compatible}, or \emph{truncated-power-law-compatible} according to the observed evidence, without interpreting heavy-tailed usage alone as criticality.

\paragraph{Relation to recent agent-memory systems.}
Table~\ref{tab:related_work_comparison} compares CTMC with closely related agent-memory systems across memory construction and retrieval, explicit context-budget allocation, tail-aware representation, distributional diagnosis, and matched-replay analysis.

\section{Problem Setup and Tail Audit}
\label{sec:preliminaries}

We consider a long-horizon agent equipped with an external memory and retrieval mechanism \citep{lewis2020rag,khandelwal2020knnlm,borgeaud2022retro,packer2023memgpt}. At interaction step $t$, the agent observes $o_t$, selects an action $a_t$, receives $o_{t+1}$, and extracts the corresponding structured states $s_t$ and $s_{t+1}$. The transition $e_t=(s_t,a_t,s_{t+1})$ is stored in a finite memory $W_t$, from which the agent retrieves context $C_t\subseteq W_t$ to estimate the next world state $\hat{s}_{t+1}$. The state-prediction error is
\begin{equation}
E_t=\mathbf{1}\left\{\hat{s}_{t+1}\neq s_{t+1}\right\}.
\label{eq:state_prediction_error}
\end{equation}

Let $\mathcal{S}_T=\{s_t:1\le t\le T\}$ denote the set of states visited over a trajectory of length $T$. For each $s\in\mathcal{S}_T$, its empirical visit count is
\begin{equation}
N_T(s)=\sum_{t=1}^{T}\mathbf{1}\left\{s_t=s\right\}.
\label{eq:state_visit_count}
\end{equation}
Order the states in $\mathcal{S}_T$ by increasing $N_T(s)$ and let $\mathcal{R}_T\subseteq\mathcal{S}_T$ contain the lowest-ranked $\lfloor|\mathcal{S}_T|/2\rfloor$ states, with ties resolved by a fixed ordering. We define tail error as
\begin{equation}
E_{\rm tail}(T)=\frac{\sum_{t=1}^{T}E_t\mathbf{1}\left\{s_t\in\mathcal{R}_T\right\}}{\sum_{t=1}^{T}\mathbf{1}\left\{s_t\in\mathcal{R}_T\right\}}.
\label{eq:tail_error}
\end{equation}
Thus, \eqref{eq:tail_error} is the empirical prediction-error rate over visits to the bottom half of states by visit frequency.

\paragraph{Statistical tail audit.}
We audit the upper tail of the empirical visit counts $N_T(s)$. Following \citet{clauset2009powerlaw} and finite-sample cautions for heavy-tailed data \citep{newman2005powerlaws,stumpf2012critical,virkar2014binned}, we estimate a lower cutoff $\hat{N}_{\min}$ and tail exponent $\hat{\alpha}$.

\begin{definition}[Fitted memory tail]
\label{def:fitted_memory_tail}
Given the estimated cutoff $\hat{N}_{\min}$, the number of observations in the fitted memory tail is
\begin{equation}
n_{\rm tail}=\sum_{s\in\mathcal{S}_T}\mathbf{1}\left\{N_T(s)\ge\hat{N}_{\min}\right\}.
\label{eq:tail_sample_size}
\end{equation}
\end{definition}

We retain only fitted tails with $n_{\rm tail}\ge100$ for statistical analysis and regard a fitted truncated power law as compatible with the observed tail when its bootstrap goodness-of-fit satisfies $p_{\rm gof}\ge0.10$. This criterion indicates statistical compatibility rather than a unique distributional explanation and does not imply scale-free topology or critical dynamics \citep{broido2019scale,sethna2001crackling}.

\paragraph{Memory ranking.}
At step $t$, the memory $W_t$ contains $M_t=|W_t|$ candidate items. We order these items by retrieval priority and assign each item a rank $r\in\{1,\ldots,M_t\}$, where $r=1$ denotes the highest-priority item and larger $r$ denotes progressively lower retrieval priority. This ranking separates frequently prioritized memory into higher-ranked items and less frequently prioritized memory into lower-ranked items. The memory controller in Section~\ref{sec:method} uses the rank $r$ to allocate memory, with the concentration parameter $\tau>0$ controlling how strongly the allocation favors higher-ranked items.

\section{Core--Tail Memory}
\label{sec:method}

\begin{figure*}[!ht]
  \centering
  \includegraphics[width=0.9\linewidth]{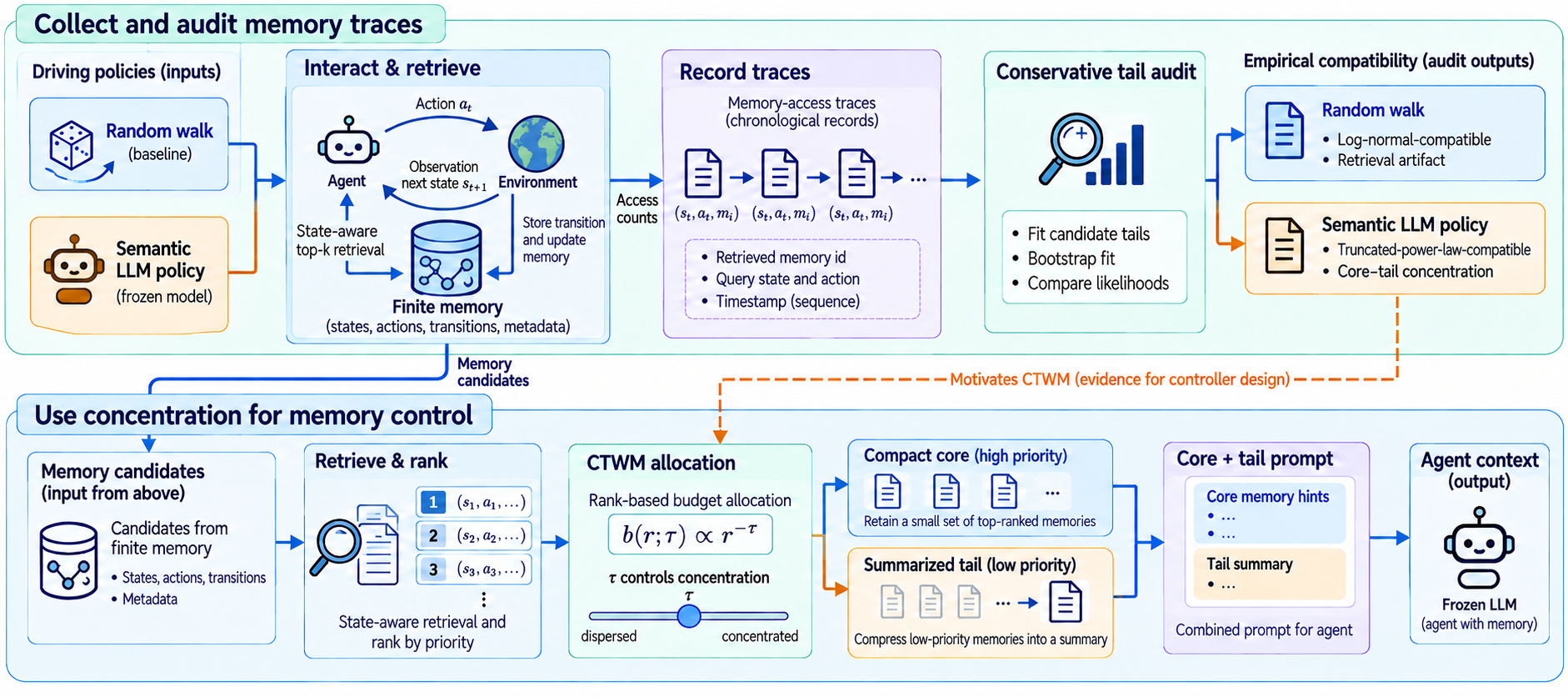}

  \caption{\textbf{Memory audit and CTMC allocation pipeline.} Memory-access traces are collected under random-walk and semantic LLM policies and audited for their tail behavior. The observed concentration then motivates CTMC, which ranks memory candidates and allocates the context budget between a compact high-priority core and a summarized tail according to $b(r;\tau)\propto r^{-\tau}$.}
  \label{fig:CTMC}
  
\end{figure*}

Figure~\ref{fig:intuition} gives an overview of CTMC within the agent loop. The controller retrieves from structured memory, constructs a core--tail context for the frozen LLM, and updates memory from subsequent interactions while recording prediction and retrieval statistics.

Figure~\ref{fig:CTMC} details the CTMC mechanism. Memory-access traces are first collected and audited for their empirical tail behavior, revealing how retrieval is concentrated across memory items. Motivated by this concentration, CTMC ranks candidate memories by retrieval priority and allocates context across ranks. High-priority items are retained individually as a compact core, while lower-priority items are compressed into a summarized tail rather than discarded. A single concentration parameter $\tau$ controls the allocation between these two parts, producing the combined context used by the agent.

\subsection{Rank-Based Allocation}
\label{sec:rank_allocation}

CTMC allocates the memory in $W_t$ according to retrieval rank $r$. For an item at rank $r$, let $\omega_t(r;\tau)$ denote its normalized share of the memory allocation at step $t$:
\begin{equation}
\omega_t(r;\tau)=\frac{r^{-\tau}}{\sum_{r'=1}^{M_t}(r')^{-\tau}}.
\label{eq:CTMC_allocation}
\end{equation}
Because smaller $r$ indicates higher retrieval priority, \eqref{eq:CTMC_allocation} assigns larger shares to higher-priority items. In particular, for $1\le r<s\le M_t$, we have $\omega_t(r;\tau)>\omega_t(s;\tau)$, so the allocation preserves the ordering produced by retrieval.

The allocation is monotone with retrieval rank. For any $1\le r<s\le M_t$ and $\tau>0$,
\begin{equation}
\omega_t(r;\tau)>\omega_t(s;\tau).
\label{eq:allocation_monotonicity}
\end{equation}
Thus, CTMC preserves the priority ordering produced by retrieval: an item with smaller rank always receives a larger allocation share. The allocation rule changes how strongly different ranks are represented in the context, but does not change their ordering.

The parameter $\tau$ controls the strength of this concentration. For any two ranks $1\le r<s\le M_t$, their relative allocation satisfies
\begin{equation}
\frac{\omega_t(r;\tau)}{\omega_t(s;\tau)}=\left(\frac{s}{r}\right)^\tau.
\label{eq:rank_concentration}
\end{equation}
Because $s/r>1$, this ratio increases with $\tau$. Smaller $\tau$ therefore produces a flatter allocation across $W_t$, giving lower-priority items a larger relative share, whereas larger $\tau$ increasingly concentrates allocation on high-priority items. In this way, $\tau$ controls memory concentration continuously without altering the underlying retrieval ranking.

\subsection{Core--Tail Memory Control}
\label{sec:core_tail_control}

CTMC converts the ranked memory into two complementary context representations. At step $t$, let $K_t<M_t$ denote the number of memory items retained individually. The first $K_t$ ranks form the core: these high-priority items are preserved as compact records containing their original state and transition information. The remaining $M_t-K_t$ items form the tail. Rather than inserting these lower-priority items separately into the context, CTMC aggregates their information into a shared summary. The resulting context therefore combines detailed evidence from the core with compressed evidence from the broader memory trace.

The allocation in \eqref{eq:CTMC_allocation} controls the relative emphasis placed on these two representations. The total allocation assigned to the core is
\begin{equation}
C_t(K_t,\tau)=\sum_{r=1}^{K_t}\omega_t(r;\tau).
\label{eq:core_mass}
\end{equation}
The tail receives the complementary allocation $1-C_t(K_t,\tau)$. CTMC uses these two quantities as the relative context budgets for the explicit core records and the compressed tail summary. Thus, $K_t$ determines the boundary between individually retained and summarized memories, while $\tau$ determines how the available context budget is distributed across that boundary.

These two controls play different roles. Changing $K_t$ changes the representation itself by moving items between the explicit core and summarized tail. In contrast, changing $\tau$ leaves the ranking and core--tail boundary fixed but changes their relative allocation. For fixed $K_t$, $C_t(K_t,\tau)$ increases monotonically with $\tau$: larger $\tau$ assigns more of the context budget to high-priority core items, whereas smaller $\tau$ preserves more budget for information from the tail. CTMC therefore adjusts memory concentration without discarding the lower-ranked portion of the memory.

When $\tau$ approaches zero, the allocation becomes increasingly uniform across ranks, whereas larger $\tau$ places progressively greater emphasis on high-priority memory. CTMC therefore spans a continuum from broadly distributed memory use to a strongly core-focused representation, while retaining lower-priority information through tail compression. This behavior can be characterized by two properties of the allocation rule. The first establishes that, for a fixed core--tail boundary, increasing $\tau$ always increases the total allocation assigned to the core.

\begin{lemma}[Monotone core allocation]
\label{lem:core_mass}
For fixed $1\le K_t<M_t$, the core allocation $C_t(K_t,\tau)$ is strictly increasing in $\tau$:
\begin{equation}
\frac{\partial}{\partial\tau}C_t(K_t,\tau)>0.
\label{eq:core_mass_monotonicity}
\end{equation}
\end{lemma}

Lemma~\ref{lem:core_mass} formalizes the concentration effect of $\tau$ at the core level: increasing $\tau$ shifts allocation toward the first $K_t$ ranks without changing either the retrieval ordering or the core--tail boundary. The proof is given in Appendix~\ref{app:proof_core_mass}. CTMC has a well-defined non-concentrated limit in which retrieval rank no longer affects allocation.

\begin{lemma}[Uniform allocation limit]
\label{lem:uniform_limit}
For fixed finite $M_t$ and every $r\in\{1,\ldots,M_t\}$,
\begin{equation}
\lim_{\tau\downarrow0}\omega_t(r;\tau)=\frac{1}{M_t}.
\label{eq:uniform_limit}
\end{equation}
\end{lemma}

Lemma~\ref{lem:uniform_limit} shows that CTMC approaches uniform allocation as $\tau$ approaches zero, with the core receiving $K_t/M_t$ of the total allocation. This limit also implies vanishing allocation inequality: because every rank receives the same asymptotic share $1/M_t$, the Gini coefficient of the allocation vector approaches zero. Together with Lemma~\ref{lem:core_mass}, this characterizes $\tau$ as a continuous concentration control, ranging from uniform allocation at small $\tau$ to increasingly core-focused allocation as $\tau$ grows. The proofs of Lemma~\ref{lem:uniform_limit} is given in Appendix~\ref{app:proof_uniform_limit}.

\subsection{Prompt encoding and Agent Loop}
\label{sec:CTMC_encoding}

Allocation and encoding play different roles in CTMC. Equation~\eqref{eq:CTMC_allocation} determines how the available context budget is distributed across retrieval ranks, while encoding determines how the allocated information is presented to the frozen LLM. As illustrated in Figures~\ref{fig:CTMC} and~\ref{fig:intuition}, high-priority memories are retained as compact core records, whereas lower-priority memories are compressed into a shared tail summary. These two memory representations are combined with the current state and action to form the context supplied to the LLM.

The interaction loop in Figure~\ref{fig:intuition} separates prediction evaluation from memory update. At step $t$, the frozen LLM acts from the current core--tail context, the environment returns the next observation, and the resulting state prediction is evaluated using \eqref{eq:state_prediction_error} before the new transition is written to $W_t$. Retrieval counts, prediction errors, and prompt-token usage are recorded at the same step, preventing the subsequent memory update from affecting the current evaluation.

Encoding can change the behavior of the frozen LLM even when the retrieval mechanism is unchanged, causing different actions and subsequent state trajectories. We separate this trajectory effect from retrieval changes through frozen-trajectory matched replay, where the compared encodings are evaluated on the same state sequence and retrieved memory identities under the same deterministic retrieval map. We compare the tail error in \eqref{eq:tail_error} between live interaction and frozen-trajectory replay: if a gap observed during live interaction disappears when the trajectory is fixed, we attribute it to the policy-induced trajectory rather than to a change in retrieval logic.

\section{Experiments}
\label{sec:experiments}

\subsection{Experiment Setup.}\label{sec:setup}
Synthetic Graph World contains six graph families with $|V|\in\{100,500,1000\}$, and the random-walk audit covers 54 replicated graph--size cells. Heavy-tail fitting and goodness-of-fit tests follow standard power-law analysis procedures \citep{clauset2009powerlaw,alstott2014powerlaw}. The semantic-policy audit uses a frozen API model with ReAct-style action selection \citep{yao2022react} across uniform-degree, exponential-degree, modular, and scale-free topologies \citep{newman2005powerlaws}. We compare CTMC ($\tau=1.0$) with seven memory baselines spanning full history, sliding-window context \citep{xiao2024streamingllm}, flat retrieval \citep{lewis2020rag}, frequency and recency caches, hierarchical summarization \citep{edge2024graphrag}, and AriGraph-style Graph Memory \citep{anokhin2024arigraph}. Cross-benchmark evaluation uses ALFWorld through its TextWorld-compatible interface \citep{cote2018textworld,shridhar2020alfworld} and LongMemEval \citep{wu2024longmemeval}. All token counts are actual API-reported prompt tokens per API-calling step.

\subsection{Random-walk control}

\paragraph{Random-walk baseline.}
We first consider random walks as a non-semantic control. Across all 54 graph--size cells, the pure power-law pass rate is 0\%, and the PSD exponent ranges from $0.048$ to $0.176$. As shown in Table~\ref{tab:rw_fss}, top-tail dispersion increases with graph size in five of six graph families. Additional concentration statistics across graph families and sizes are reported in Appendix~\ref{app:rw_concentration}. Thus, random walks produce finite-size concentration but do not exhibit the heavy-tailed behavior studied below.

\paragraph{Effect of retrieval.}
We further examine the source of this concentration. State-aware top-$k$
retrieval yields Gini values from $0.44$ to $0.63$, whereas disabling
retrieval reduces them to $0.29$--$0.35$, corresponding to a 33--45\%
decrease. Non-preferential retrieval retains 92--94\% of the scale-free
Gini. Additional retrieval-design controls are reported in
Appendix~\ref{app:retrieval_controls}. These controls show that finite
retrieval itself creates access concentration, providing a baseline for
evaluating the additional effect of semantic policies.

\begin{table}[!ht]

\centering\small
\caption{Random-walk control across graph sizes. 
$\sigma_{\mathrm{top10}}$ is the standard deviation of access counts among the top 10\% most accessed states.}
\label{tab:rw_fss}

\resizebox{\linewidth}{!}{%
\begin{tabular}{lccccc}
\toprule
Graph family & $\sigma_{100}$ & $\sigma_{500}$ & $\sigma_{1000}$ & Slope & $R^2$ \\
\midrule
uniform & 0.152 & 0.802 & 1.364 & 0.505 & 0.967 \\
exponential & 0.333 & 0.900 & 1.369 & 0.433 & 0.971 \\
scale-free & 0.385 & 0.809 & 1.349 & 0.392 & 0.915 \\
modular & 0.162 & 1.461 & 1.366 & 0.573 & 0.874 \\
mixed & 0.231 & 1.252 & 1.430 & 0.541 & 0.975 \\
baseline-symmetric & 0.253 & 1.190 & 1.543 & 0.564 & 0.999 \\
\bottomrule
\end{tabular}
}
\end{table}

\subsection{Heavy-tail test under semantic policies}

We next test whether semantic LLM policies induce heavy-tailed memory
access beyond the concentration observed in the random-walk control.
Figure~\ref{fig:access_ccdf} directly compares the empirical access
tails under the two policies. While both exhibit non-uniform access,
the semantic-policy trace shows a pronounced tail above the fitted
cutoff, motivating a formal tail audit.

\begin{figure*}[!ht]
\centering
\includegraphics[width=0.8\linewidth]{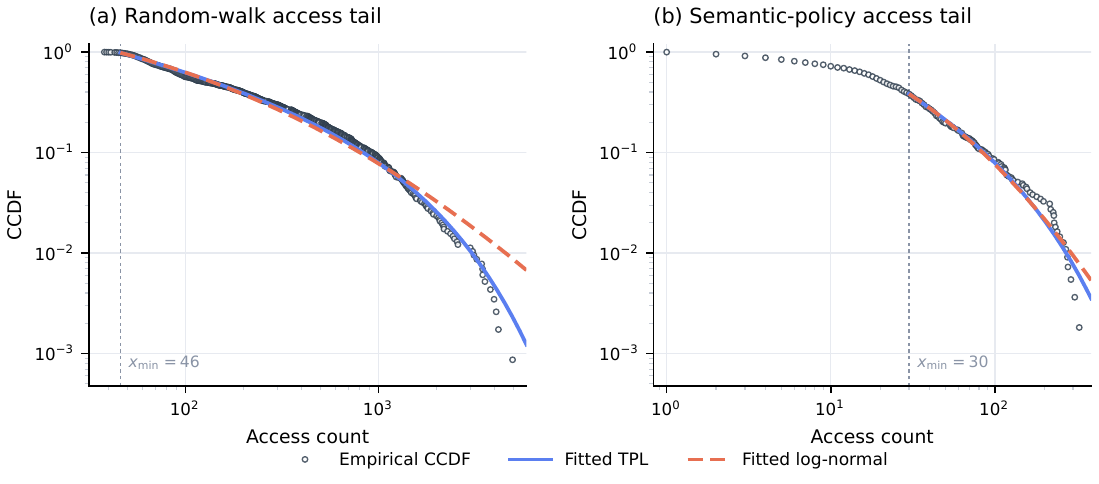}

\caption{Empirical CCDFs of memory access under random-walk and semantic
policies, with fitted truncated-power-law (TPL) and log-normal tails
above the estimated $x_{\min}$.}
\label{fig:access_ccdf}

\end{figure*}

The tail behavior is reproducible across runs. Across three scale-free
runs, the fitted exponents remain stable at $\hat{\alpha}=2.348$,
$2.412$, and $2.390$. Two runs pass the bootstrap screen
($p_{\rm gof}=0.570$ and $0.270$), while the third is borderline
($p_{\rm gof}=0.080$); all three favor a truncated power law over the
log-normal alternative. We next test whether this pattern depends on a
particular graph structure. Table~\ref{tab:llm} and
Figure~\ref{fig:topology_evidence} show that scale-free, modular, and
exponential graphs pass the tail-fit screen, whereas the uniform graph
does not. Together, these results show that heavy-tailed memory access
under semantic policies is reproducible across runs and persists across
multiple graph topologies.

\begin{table}[!ht]
\centering\small
\caption{Heavy-tail audit across graph topologies under semantic LLM
policies. $\hat{\alpha}$ is the fitted tail exponent, $p_{\rm gof}$ is
the bootstrap goodness-of-fit $p$-value, and the final column reports
the likelihood-ratio statistic comparing TPL with the log-normal
alternative, where positive values favor TPL. A topology passes the
tail-fit screen when $p_{\rm gof}\geq0.10$.}
\label{tab:llm}

\resizebox{0.8\linewidth}{!}{%
\begin{tabular}{lccc}
\toprule
Topology & $\hat{\alpha}$ & $p_{\rm gof}$ & TPL vs.\ log-normal \\
\midrule
scale-free & 2.35 & 0.57 & 3.22 \\
modular & 2.03 & 0.96 & 0.60 \\
exponential & 2.98 & 0.36 & 0.81 \\
uniform & 2.91 & 0.06 & 1.69 \\
\bottomrule
\end{tabular}
}

\end{table}

\begin{figure*}[!ht]
\centering
\includegraphics[width=0.8\linewidth]{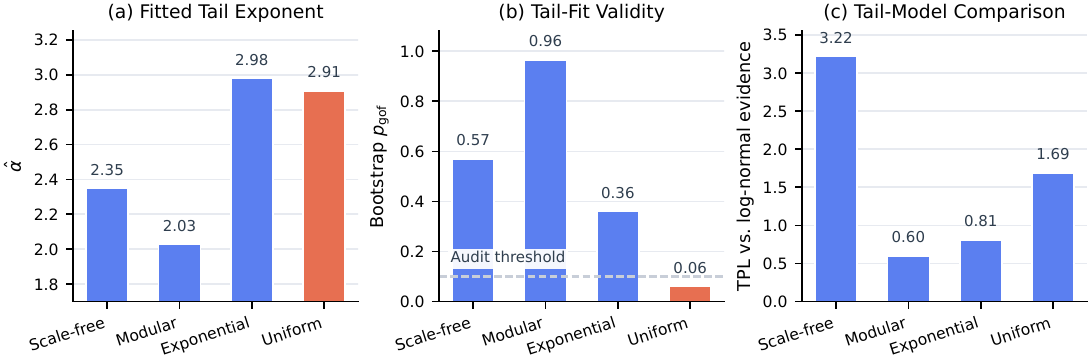}

\caption{Heavy-tail audit across graph topologies under semantic LLM
policies. Left: fitted tail exponent $\hat{\alpha}$. Middle: bootstrap
goodness-of-fit, with $p_{\rm gof}=0.10$ marking the audit threshold.
Scale-free, modular, and exponential graphs pass the screen, whereas
uniform graphs do not. Right: evidence favoring a truncated power law
over a log-normal tail, strongest for scale-free graphs.}
\label{fig:topology_evidence}

\end{figure*}

Beyond tail shape, the resulting access distribution is also strongly
uneven. Figure~\ref{fig:memory_concentration} shows that a small fraction
of memories accounts for a disproportionate share of total access; for
example, the top 10\% captures 47.7\% of random-walk access and 43.8\%
under the semantic policy. This concentration motivates using observed
memory usage to allocate a finite context budget rather than treating
all memories uniformly.

\begin{figure*}[!ht]
\centering
\includegraphics[width=0.8\linewidth]{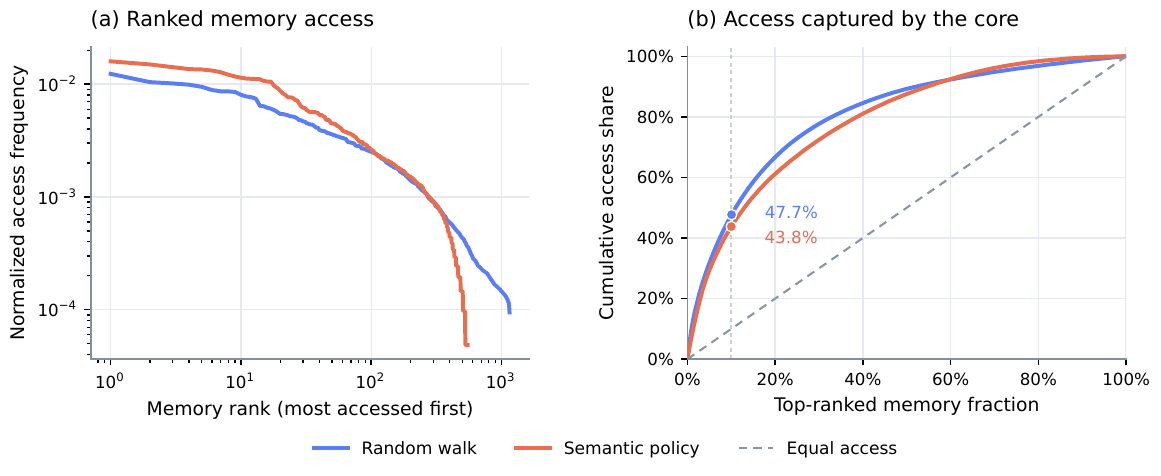}

\caption{Memory-access concentration under random-walk and semantic
policies. Left: normalized access frequency ordered from most to least
accessed memory. Right: cumulative access captured by the top-ranked
memory fraction.}
\label{fig:memory_concentration}

\end{figure*}

\subsection{Heavy-tail-aware memory allocation}

We next test whether the observed heavy-tail structure provides a useful signal for memory allocation. Table~\ref{tab:baselines} compares CTMC with seven memory baselines, where CTMC achieves both the lowest tail error and the lowest token cost. In Table~\ref{tab:baselines}, CTMC reduces tokens per step from $170.5$ to $155.8$ and tail error from $0.948$ to $0.778$ relative to Graph Memory. Across three runs, CTMC further reduces token use by 5.89\% and tail error by 13.6\%, with lower token use in all three runs. These results show that heavy-tail-aware allocation consistently improves both context efficiency and tail-state prediction.

\begin{table}[!ht]
\centering\small
\caption{Memory-policy comparison on Synthetic Graph World. Tokens/step denotes the average number of prompt tokens used per API-calling step, and tail error measures prediction error on the bottom-50\% least-accessed states.}
\label{tab:baselines}

\resizebox{0.8\linewidth}{!}{%
\begin{tabular}{lcc}
\toprule
Method & Tokens/step & Tail error \\
\midrule
Full History & 3601.1 & 1.000 \\
Sliding Window & 1479.5 & 0.987 \\
Flat Retrieval & 159.9 & 0.994 \\
Frequency Cache & 166.4 & 1.000 \\
Recency Cache & 159.0 & 1.000 \\
Hierarchical Summary & 411.7 & 0.996 \\
Graph Memory & 170.5 & 0.948 \\
CTMC ($\tau=1.0$) & \textbf{155.8} & \textbf{0.778} \\
\bottomrule
\end{tabular}
}

\end{table}

\subsection{Ablation Study}
\label{sec:ablation}

\paragraph{Allocation and encoding.}
We separate CTMC allocation from compact encoding in Table~\ref{tab:ablation}. CTMC reduces tail error under both encodings, while compact encoding primarily reduces token cost. Compact CTMC achieves 155.7 tokens per step with the lowest tail error of 0.805, compared with 156.4 tokens and 0.931 tail error for compact Graph Memory. Under frozen-trajectory replay with retrieval fixed, the verbose--compact tail-error gap decreases to within $4\times10^{-4}$, further indicating that the main tail-error improvement comes from allocation rather than encoding.

\begin{table}[!ht]
\centering\small
\caption{Ablation of memory allocation and encoding. CTMC reduces
tail error under both encodings, while compact encoding primarily
reduces token cost.}
\label{tab:ablation}

\resizebox{0.8\linewidth}{!}{%
\begin{tabular}{llcc}
\toprule
Method & Encoding & Tokens/step & Tail error \\
\midrule
Graph Memory & Verbose & 170.7 & 0.932 \\
CTMC & Verbose & 200.3 & 0.856 \\
\midrule
Graph Memory & Compact & 156.4 & 0.931 \\
CTMC & Compact & \textbf{155.7} & \textbf{0.805} \\
\bottomrule
\end{tabular}%
}

\end{table}

\paragraph{Memory allocation strategy.}
\begin{table}[!ht]
\centering\small
\caption{Memory-allocation ablation under matched context budgets.}
\label{tab:allocation_ablation}

\resizebox{0.8\linewidth}{!}{%
\begin{tabular}{lccc}
\toprule
Allocation & Tokens/step & Overall error & Tail error \\
\midrule
CTMC ($\tau=1.0$) & \textbf{155.7} & \textbf{0.214} & \textbf{0.805} \\
Uniform ($\tau=0$) & 156.0 & 0.236 & 0.861 \\
Shuffled ranks ($\tau=1.0$) & 155.9 & 0.249 & 0.889 \\
Reversed ranks ($\tau=1.0$) & 156.1 & 0.267 & 0.917 \\
\bottomrule
\end{tabular}
}
\end{table}

We test whether CTMC benefits from the observed memory-usage structure.
Table~\ref{tab:allocation_ablation} compares three controls:
\begin{itemize}
    \item \textbf{Uniform} ($\tau=0$): removes usage-dependent allocation.
    \item \textbf{Shuffled}: randomly permutes the observed usage order.
    \item \textbf{Reversed}: inverts the observed usage order.
\end{itemize}
CTMC achieves the lowest overall and tail errors, while the controls
progressively degrade performance under nearly identical token use
(155.7--156.1 per step). This shows that the gain comes from preserving
the observed usage structure rather than additional context.

\subsection{Control by allocation exponent}

We next test the monotonic concentration predicted by Lemma~\ref{lem:core_mass}. Figure~\ref{fig:tau_control} sweeps $\tau$ from $0.25$ to $2.0$ across three seeds. Increasing $\tau$ consistently increases Gini ($\rho=0.976$), skew ($\rho=0.988$), and the max/median ratio ($\rho=1.000$). These results empirically confirm that $\tau$ provides direct control over memory concentration, with larger values shifting allocation toward highly ranked items.

\begin{figure}[!ht]
\centering
\includegraphics[width=0.9\linewidth]{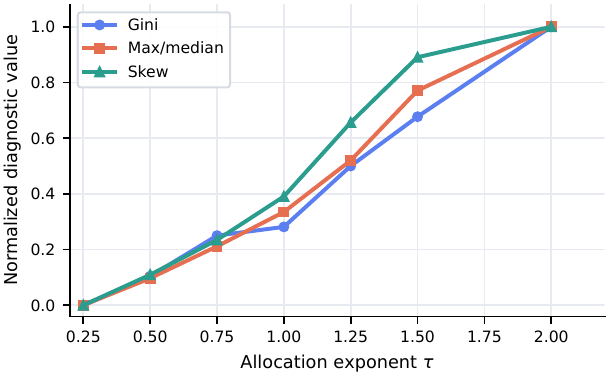}

\caption{Effect of $\tau$ on memory concentration across three seeds. Gini, skew, and max/median ratio increase consistently with $\tau$.}
\label{fig:tau_control}

\end{figure}

\subsection{Cross-benchmark evaluation}

We finally test whether CTMC's context savings extend beyond Synthetic
Graph World. As shown in Figure~\ref{fig:crossbench}, CTMC reduces
prompt tokens in every paired configuration: 5.68--6.22\% across
synthetic replications, 3.90--6.61\% across ALFWorld
settings~\citep{cote2018textworld,shridhar2020alfworld}, and 24.48\%
on LongMemEval~\citep{wu2024longmemeval} while maintaining aggregate
accuracy parity. These results demonstrate consistent context efficiency
across graph-based, embodied, and long-term conversational memory
settings.

\begin{figure}[!ht]
\centering
\includegraphics[width=0.95\linewidth]{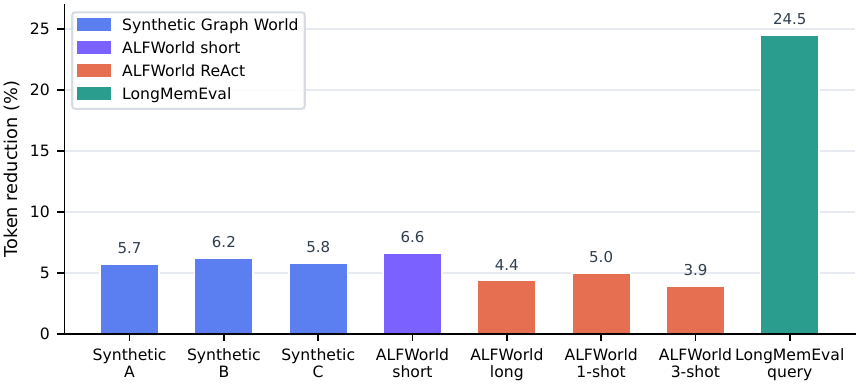}

\caption{Prompt-token reduction of CTMC relative to Graph Memory across
paired configurations.}
\label{fig:crossbench}

\end{figure}

We further break down LongMemEval by question type in
Figure~\ref{fig:longmemeval_types}. CTMC consistently reduces prompt
cost across question categories while maintaining comparable accuracy,
showing that the aggregate saving is not driven by a single question
type.

\begin{figure}[!ht]
\centering
\includegraphics[width=\linewidth]{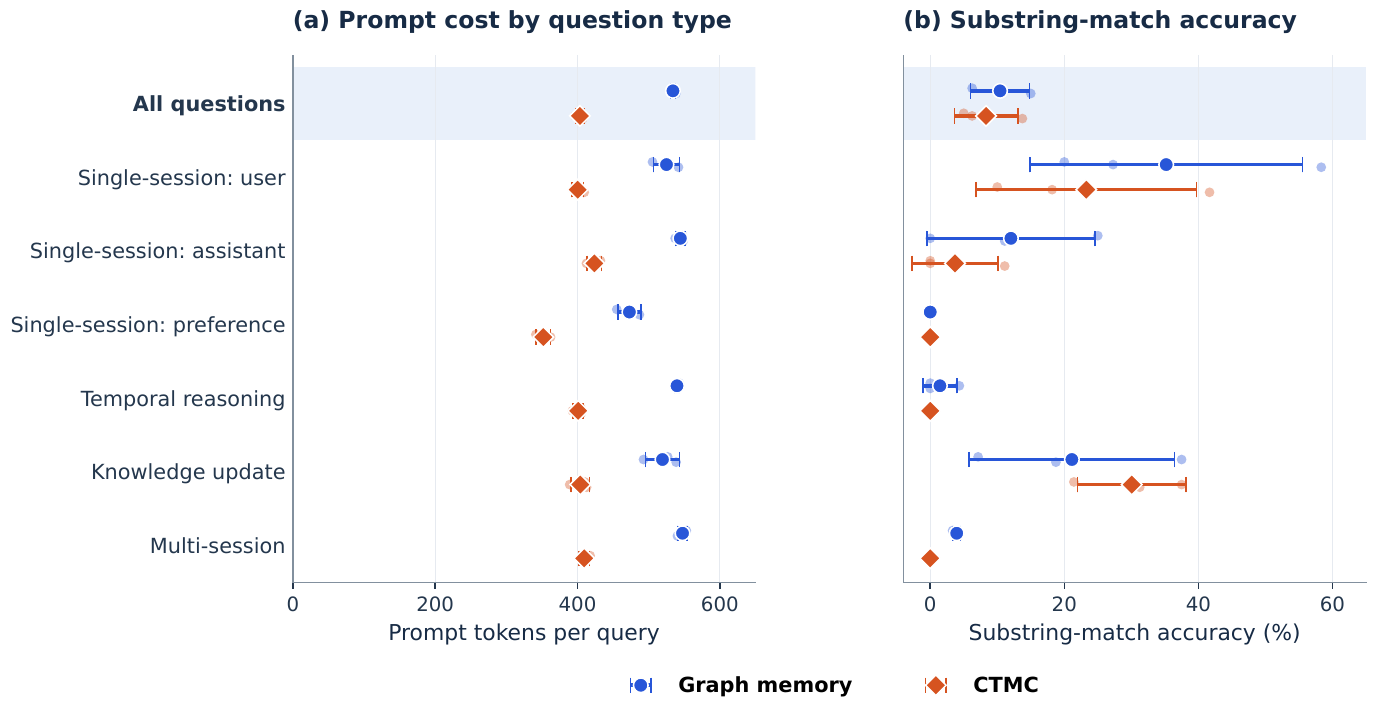}

\caption{LongMemEval results by question type. Left: prompt tokens per
query. Right: substring-match accuracy. Error bars show seed-level
mean $\pm$ sample standard deviation.}
\label{fig:longmemeval_types}

\end{figure}

\section{Conclusion}
\label{sec:conclusion}

We studied memory access in long-horizon language agents and found that usage can become highly concentrated, while careful random-walk and retrieval controls show that concentration alone is insufficient evidence for semantic heavy-tail behavior. Our tail audits nevertheless identify reproducible heavy-tailed access under semantic policies, motivating CTMC, which converts observed usage structure into rank-based context allocation. Across Synthetic Graph World, ALFWorld, and LongMemEval, CTMC improves the context--performance tradeoff, with ablations confirming that the gains arise primarily from memory allocation rather than encoding. These results establish memory-access structure as a measurable and practically useful signal for managing finite context in long-horizon agents.

\bibliography{references}
\bibliographystyle{ACM-Reference-Format}

\clearpage
\appendix

\section{Additional Random-Walk Concentration Results}
\label{app:rw_concentration}

We provide additional results for the random-walk control to characterize
how memory-access concentration changes with graph size.
Table~\ref{tab:rw_concentration} reports three complementary statistics
across six graph families. The Gini coefficient measures overall access
inequality, the Top-10\% share measures the fraction of accesses captured
by the most frequently accessed 10\% of memories, and
$\sigma_{\mathrm{top10}}$ measures dispersion within this highly accessed
subset. Across graph families, both Gini and Top-10\% share generally
increase substantially from 100 to 1,000 nodes. The scaling of
$\sigma_{\mathrm{top10}}$ is also positive for all six families, with
$R^2$ ranging from $0.874$ to $0.999$. These results show that strong
access concentration can emerge under a non-semantic random-walk policy
and becomes more pronounced as the environment grows. Combined with the
failure of all 54 random-walk graph--size cells to pass the pure
power-law screen, this result reinforces that access concentration alone
is insufficient evidence for the semantic-policy tail behavior studied
in the main text.

\begin{table*}[!ht]
\centering
\caption{Random-walk memory-access concentration across graph families
and graph sizes. Gini measures overall access inequality, Top-10\% share
is the percentage of accesses captured by the most accessed 10\% of
memories, and $\sigma_{\mathrm{top10}}$ is the standard deviation of
access counts within this top group. Columns 100, 500, and 1,000 denote
the number of graph nodes. Slope and $R^2$ summarize the scaling of
$\sigma_{\mathrm{top10}}$ with graph size.}
\label{tab:rw_concentration}
\footnotesize
\setlength{\tabcolsep}{2.5pt}
\resizebox{0.85\textwidth}{!}{%
\begin{tabular}{lccccccccccc}
\toprule
& \multicolumn{3}{c}{Gini coefficient}
& \multicolumn{3}{c}{Top-10\% share (\%)}
& \multicolumn{3}{c}{$\sigma_{\mathrm{top10}}$}
& \multicolumn{2}{c}{Scaling} \\
\cmidrule(lr){2-4}
\cmidrule(lr){5-7}
\cmidrule(lr){8-10}
\cmidrule(lr){11-12}
Graph family
& 100 & 500 & 1{,}000
& 100 & 500 & 1{,}000
& 100 & 500 & 1{,}000
& Slope & $R^2$ \\
\midrule
Uniform degree
& 0.5741 & 0.7550 & 0.7934
& 39.53 & 74.84 & 78.57
& 0.152 & 0.802 & 1.364
& 0.505 & 0.967 \\

Exponential degree
& 0.6054 & 0.7618 & 0.7954
& 46.46 & 75.77 & 78.68
& 0.333 & 0.900 & 1.369
& 0.433 & 0.971 \\

Scale-free
& 0.6045 & 0.7563 & 0.7944
& 46.24 & 74.70 & 78.63
& 0.385 & 0.809 & 1.349
& 0.392 & 0.915 \\

Modular
& 0.4851 & 0.8098 & 0.8170
& 30.85 & 79.56 & 80.48
& 0.162 & 1.461 & 1.366
& 0.573 & 0.874 \\

Mixed
& 0.4024 & 0.7867 & 0.8220
& 26.66 & 78.13 & 80.32
& 0.231 & 1.252 & 1.430
& 0.541 & 0.975 \\

Symmetric
& 0.6141 & 0.7515 & 0.7868
& 50.44 & 76.98 & 78.58
& 0.253 & 1.190 & 1.543
& 0.564 & 0.999 \\
\bottomrule
\end{tabular}%
}
\end{table*}

\section{Retrieval-Design Controls}
\label{app:retrieval_controls}

We further examine whether the access concentration observed under the
random-walk policy can arise from the retrieval mechanism itself.
Table~\ref{tab:retrieval_design_controls} compares four retrieval
designs while matching the underlying state and transition summaries.
Across all five graph families, disabling retrieval produces the lowest
Gini coefficient and generally the lowest max/median ratio. Introducing
retrieval increases concentration even without a semantic policy, with
state-aware frequency retrieval producing the strongest concentration
in most settings. For example, on the scale-free graph, the Gini
coefficient increases from $0.350$ without retrieval to $0.633$ with
state-aware frequency retrieval, while the max/median ratio increases
from $5.25$ to $38.00$. These results show that retrieval design itself
can induce substantial access concentration and further motivate
separating concentration from statistical evidence of semantic-policy
heavy-tail behavior.

\begin{table*}[!ht]
\centering
\caption{Retrieval-design controls under the random-walk policy. Each
entry is a single-seed measurement. State and transition summaries are
matched across retrieval conditions. Gini measures overall inequality
in memory access, while Max/median is the ratio between the maximum and
median memory-access counts.}
\label{tab:retrieval_design_controls}
\small
\setlength{\tabcolsep}{6pt}
\begin{tabular}{lccccc}
\toprule
Retrieval method & Uniform degree & Exponential degree & Scale-free & Modular & Mixed \\
\midrule
\multicolumn{6}{l}{\textit{Gini coefficient}} \\
No retrieval & 0.334 & 0.339 & 0.350 & 0.333 & 0.292 \\
Recency & 0.577 & 0.569 & 0.584 & 0.472 & 0.358 \\
State-mate uniform & 0.571 & 0.574 & 0.595 & 0.468 & 0.351 \\
State-aware frequency & 0.584 & 0.608 & 0.633 & 0.500 & 0.435 \\
\midrule
\multicolumn{6}{l}{\textit{Max/median}} \\
No retrieval & 5.50 & 4.00 & 5.25 & 4.29 & 3.18 \\
Recency & 12.88 & 17.70 & 17.25 & 6.25 & 3.88 \\
State-mate uniform & 10.56 & 15.89 & 17.00 & 6.12 & 3.38 \\
State-aware frequency & 12.88 & 20.33 & 38.00 & 7.41 & 7.14 \\
\bottomrule
\end{tabular}
\end{table*}

\section{Proofs for Core--Tail Memory Cotroller}
\label{app:formal_proofs}

\subsection{Proof of Lemma~\ref{lem:core_mass}}
\label{app:proof_core_mass}

\begin{proof}
Define the allocation-weighted mean log-rank
\begin{equation}
\mu_t(\tau)=\sum_{r=1}^{M_t}\omega_t(r;\tau)\log r.
\label{eq:mean_log_rank}
\end{equation}
Differentiating \eqref{eq:CTMC_allocation} with respect to $\tau$ gives
\begin{equation}
\frac{\partial}{\partial\tau}\omega_t(r;\tau)=\omega_t(r;\tau)\left(\mu_t(\tau)-\log r\right).
\label{eq:allocation_derivative}
\end{equation}
Using \eqref{eq:core_mass},
\begin{equation}
\begin{aligned}
\frac{\partial}{\partial\tau}C_t(K_t,\tau) &= \sum_{r=1}^{K_t}\frac{\partial}{\partial\tau}\omega_t(r;\tau) \\
&= \sum_{r=1}^{K_t}\omega_t(r;\tau)\left(\mu_t(\tau)-\log r\right) \\
&= C_t(K_t,\tau)\mu_t(\tau)-\sum_{r=1}^{K_t}\omega_t(r;\tau)\log r.
\end{aligned}
\label{eq:core_mass_derivative}
\end{equation}
Define the allocation-weighted mean log-rank within the core as
\begin{equation}
\mu_{{\rm core},t}(\tau)=\frac{\sum_{r=1}^{K_t}\omega_t(r;\tau)\log r}{C_t(K_t,\tau)}.
\label{eq:core_mean_log_rank}
\end{equation}
Substituting \eqref{eq:core_mean_log_rank} into \eqref{eq:core_mass_derivative} yields
\begin{equation}
\begin{aligned}
\frac{\partial}{\partial\tau}C_t(K_t,\tau) &= C_t(K_t,\tau)\mu_t(\tau)-C_t(K_t,\tau)\mu_{{\rm core},t}(\tau) \\
&= C_t(K_t,\tau)\left(\mu_t(\tau)-\mu_{{\rm core},t}(\tau)\right).
\end{aligned}
\label{eq:core_mass_derivative_compact}
\end{equation}
Because every rank outside the core is larger than every rank inside the core, the allocation-weighted mean of $\log r$ over the tail is strictly larger than $\mu_{{\rm core},t}(\tau)$. Since $K_t<M_t$ and $\omega_t(r;\tau)>0$ for all ranks, the full mean $\mu_t(\tau)$ is a strict convex combination of the core and tail means. Therefore
\begin{equation}
\mu_t(\tau)>\mu_{{\rm core},t}(\tau).
\end{equation}
Since $C_t(K_t,\tau)>0$, \eqref{eq:core_mass_derivative_compact} implies
\begin{equation}
\frac{\partial}{\partial\tau}C_t(K_t,\tau)>0.
\end{equation}
Hence $C_t(K_t,\tau)$ is strictly increasing in $\tau$.
\end{proof}

\subsection{Proof of Lemma~\ref{lem:uniform_limit}}
\label{app:proof_uniform_limit}

\begin{proof}
For any fixed $r\in\{1,\ldots,M_t\}$,
\begin{equation}
\begin{aligned}
\lim_{\tau\downarrow0}r^{-\tau} &= \lim_{\tau\downarrow0}\exp(-\tau\log r) \\
&= 1.
\end{aligned}
\end{equation}
Because $M_t$ is finite,
\begin{equation}
\begin{aligned}
\lim_{\tau\downarrow0}\sum_{r'=1}^{M_t}(r')^{-\tau} &= \sum_{r'=1}^{M_t}\lim_{\tau\downarrow0}(r')^{-\tau} \\
&= M_t.
\end{aligned}
\end{equation}
Applying these limits to \eqref{eq:CTMC_allocation},
\begin{equation}
\begin{aligned}
\lim_{\tau\downarrow0}\omega_t(r;\tau) &= \lim_{\tau\downarrow0}\frac{r^{-\tau}}{\sum_{r'=1}^{M_t}(r')^{-\tau}} \\
&= \frac{1}{M_t}.
\end{aligned}
\end{equation}
Thus every rank receives the same allocation in the limit $\tau\downarrow0$.
\end{proof}

\section{Prompt Templates: Compact and Verbose Encoding}
\label{app:encoding_examples}

The final CTMC variant changes only the LLM-facing encoding of retrieved hints. The compact format removes raw score annotations, preserves rank order, and aggregates the low-rank tail into a short typed summary. The prompt contract below is the template used by the controller: it exposes the same memory evidence as a ranked core plus a measured tail, while asking the model to produce a bounded next-state prediction rather than free-form reasoning.

\begin{samepage}
\noindent\textbf{Compact CTMC prompt template.}\par

\noindent\fcolorbox{teal!35}{teal!6}{%
\parbox{0.96\linewidth}{\small\ttfamily
Task: predict the next structured state after the current action.\par
Current state: \{state\}\par
Action: \{action\}\par
Core memories, ranked by retrieval order:\par
rank 1: \{state\} -> \{action\} -> \{next state\}; hint=\{short evidence\}\par
rank 2: \{state\} -> \{action\} -> \{next state\}; hint=\{short evidence\}\par
...\par
Tail summary: count=\{N\}; dominant transitions=\{types\}; rare exceptions=\{short list\}\par
Budget rule: use core memories first; use tail summary only to break ties or cover rare states.\par
Output: next-state=\{prediction\}; confidence=\{low, medium, high\}; evidence-rank=\{rank or tail\}.
}}
\end{samepage}

\begin{samepage}
\noindent\textbf{Verbose versus compact examples.}\par

\noindent\fcolorbox{gray!30}{gray!6}{%
\parbox{0.96\linewidth}{\small\textbf{Synthetic Graph World, verbose graph memory}\par\ttfamily
Retrieved items: (room, pickup, bag, score 1.23, degree 8, source step 42);
(room, inspect, drawer, score 1.18, degree 6, source step 39);
(park, walk, home, score 0.87, degree 5, source step 17);
(park, jog, home, score 0.32, degree 3, source step 14).\par
Instruction: use all retrieved transitions and scores to predict the next state.
}}
\end{samepage}

\noindent\fcolorbox{teal!35}{teal!6}{%
\parbox{0.96\linewidth}{\small\textbf{Synthetic Graph World, compact CTMC}\par\ttfamily
Current: room. Action: pickup.\par
Core: rank 1 room -> pickup -> bag; rank 2 room -> inspect -> drawer.\par
Tail summary: 2 park-to-home variants; no room pickup exception observed.\par
Output next-state only; cite rank 1, rank 2, or tail.
}}

\noindent\fcolorbox{gray!30}{gray!6}{%
\parbox{0.96\linewidth}{\small\textbf{ALFWorld, verbose history}\par\ttfamily
Observation history: you are in the kitchen; cabinet closed; apple in sink;
you opened cabinet at step 12; you examined counter at step 15; you washed apple at step 19;
retrieval scores: cabinet 0.91, sink 0.88, fridge 0.76, counter 0.41.\par
Question: after action put apple in fridge, predict the next observation.
}}

\noindent\fcolorbox{teal!35}{teal!6}{%
\parbox{0.96\linewidth}{\small\textbf{ALFWorld, compact CTMC}\par\ttfamily
Current: holding washed apple in kitchen. Action: put apple in fridge.\par
Core: rank 1 fridge is open; rank 2 apple is clean; rank 3 agent is at fridge.\par
Tail summary: cabinet and counter facts are background; no blocking object retrieved.\par
Output: next observation and one evidence rank.
}}

\noindent\fcolorbox{gray!30}{gray!6}{%
\parbox{0.96\linewidth}{\small\textbf{LongMemEval, verbose session memory}\par\ttfamily
Session snippets: user mentioned moving to Boston in April; user later corrected trip to May;
assistant discussed lease renewal; another session mentioned Seattle conference;
retrieval scores: Boston 0.79, May correction 0.58, lease 0.51, Seattle 0.49.\par
Question: what month did the user finally plan to move?
}}

\noindent\fcolorbox{teal!35}{teal!6}{%
\parbox{0.96\linewidth}{\small\textbf{LongMemEval, compact CTMC}\par\ttfamily
Question type: long-term recall with update.\par
Core: rank 1 initial plan Boston in April; rank 2 later correction move in May.\par
Tail summary: lease and Seattle snippets are related context but not answer-bearing.\par
Output: answer=May; evidence-rank=2; note=correction overrides earlier memory.
}}

\paragraph{Why the template matters.}
The verbose form is useful for auditing because it exposes scores and raw provenance, but it asks the frozen model to parse many low-value fields. The compact CTMC form keeps the information that matters for control: rank order, the current state/action, a small core, and an explicit tail summary. This is why compact encoding can lower token use without merely deleting the tail; the tail remains present as an aggregate object with a clear role in the prompt.

\end{document}

%% file: aamas-2027-compat.tex
\makeatletter
\def\aamas@legacy@editors{\fi}
\ifx\acmConference@editors\aamas@legacy@editors
  \expandafter\@firstofone
\else
  \expandafter\@gobble
\fi
{\acmConference@editors\gdef\acmConference@editors{}}
\let\aamas@legacy@editors\@undefined
\makeatother